%% file: paper.tex
\documentclass[11pt]{article}

\usepackage{times}
\usepackage{scicite}

\usepackage{amsmath}
\usepackage{amsfonts}
\usepackage{amssymb}
\usepackage{xcolor}
\usepackage{amsmath,mathtools}
\usepackage{algorithm,algorithmic}
\usepackage{graphicx}
\usepackage{microtype}
\usepackage{subfigure}
\usepackage{booktabs}
\usepackage{amsmath}
\usepackage{amsthm}
\usepackage{amsfonts}%

\newcounter{lastnote}

\title{\bf Generalization Guarantees for Sparse Kernel Approximation with Entropic Optimal Features} 
\author
{Liang Ding, Rui Tuo, Shahin Shahrampour\\
\\
\normalsize{Texas A\&M University\vspace{.2cm}}\\
\normalsize{E-mail: {\tt \{ldingaa,ruituo,shahin\}@tamu.edu}}
}

\date{}

\input{header.tex}

\begin{document} 


\baselineskip24pt


\maketitle

\begin{abstract}
 Despite their success, kernel methods suffer from a massive computational cost in practice. In this paper, in lieu of commonly used kernel expansion with respect to $N$ inputs, we develop a novel optimal design maximizing the entropy among kernel features. This procedure results in a kernel expansion with respect to entropic optimal features (EOF), improving the data representation dramatically due to features dissimilarity. Under mild technical assumptions, our generalization bound shows that with only $\CalO(N^{\frac{1}{4}})$ features (disregarding logarithmic factors), we can achieve the optimal statistical accuracy (i.e., $\CalO(1/\sqrt{N})$). The salient feature of our design is its sparsity that significantly reduces the time and space cost. Our numerical experiments on benchmark datasets verify the superiority of EOF over the state-of-the-art in kernel approximation. 
\end{abstract}

\section{Introduction}
Kernel methods are powerful tools in describing nonlinear data models. However, despite their success in various machine learning tasks, kernel methods always suffer from scalability issues, especially when the learning task involves matrix inversion (e.g., kernel ridge regression). This is simply due to the fact that for a dataset of size $N$, the inversion step requires $\CalO(N^3)$ time cost. To tackle this problem, a great deal of research has been dedicated to the approximation of kernels using low-rank surrogates \cite{smola2000sparse,fine2001efficient,rahimi2008random}. By approximating the kernel, these methods deal with a linear problem, potentially solvable in a linear time with respect to $N$ (see e.g. \cite{joachims2006training} for linear Support Vector Machines (SVM)).

In the approximation of kernel with a finite number of features, one fundamental question is how to select the features. As an example, in supervised learning, we are interested to identify features that lead to low out-of-sample error. This question has been studied in the context of random features, which is an elegant method for kernel approximation \cite{rahimi2008random}.
Most of the works in this area improve the out-of-sample performance by modifying the stochastic oracle from which random features are sampled \cite{sinha2016learning,avron2017random,shahrampour2018data}. Nevertheless, these methods deal with dense feature matrices (due to randomness) and still require a large number of features to learn the data subspace. Decreasing the number of features directly affects the time and space costs, and to achieve that we must choose features that are as distinct as possible (to better span the space). Focusing on explicit features, we aim to achieve this goal in the current work.

\subsection{Our Contributions}
In this paper, we study low-rank kernel approximation by finding a set of mutually orthogonal features with nested and compact supports. We first theoretically characterize a condition (based on the Sturm-Liouville problem), which allows us to obtain such features. Then, we propose a novel optimal design method that maximizes the metric entropy among those features. The problem is formulated as a combinatorial optimization with a constraint on the number of features used for approximation. The optimization is generally NP-hard but yields closed-form solutions for specific numbers of features. The algorithm, dubbed entropic optimal features (EOF), can use these features for supervised learning. The construction properties of features (orthogonality, compact support, and nested support) result in a sparse approximation saving dramatically on time and space costs. We establish a generalization bound for EOF that shows with only $\CalO(N^{\frac{1}{4}})$ features (disregarding logarithmic factors), we can achieve the optimal statistical accuracy (i.e., $\CalO(1/\sqrt{N})$). Our numerical experiments on benchmark datasets verify the superiority of EOF over the state-of-the-art in kernel approximation. While we postpone the exhaustive literature review to Section \ref{sec:lit}, none of the previous works has approached the problem from the entropy maximization perspective, which is the unique distinction of the current work.

\section{Preliminaries on Kernel Methods}
Kernel methods map finite-dimensional data to a potentially infinite dimensional feature space. Any element $f$ in the reproducing kernel Hilbert space (RKHS) of $k$, denoted by $\CalH_k$, has the following representation:
\begin{equation}
\label{eq:kernel expansion}
\begin{aligned}
    f=\sum_{i=1}^\infty\langle f,g_i\rangle_k g_i,
\end{aligned}
\end{equation}
where $\langle\cdot,\cdot\rangle_{k}$ RKHS inner product induced by $k$ and $\{g_i\}$ is any  feature set (i.e., orthonormal basis) that spans the space $\CalH_k$. In general, the kernel trick relies on the observation that the inner product $\langle k(\cdot,\Bx),k(\cdot,\Bx')\rangle_k=k(\Bx,\Bx')$ with $\Bx,\Bx'\in\Real^D$ (reproducing property), so $k(\Bx,\Bx')$ is cheap to compute without the need to calculate the inner product. In this case, the feature set selected in equation \eqref{eq:kernel expansion} is $\{k(\cdot,\Bx):\Bx\in\Real^D\}$ and the target function can be written as $\sum_{i}c_ik(\cdot,\Bx_i)$. 



Under mild conditions, by the Representer Theorem, it is guaranteed that any solution of the risk minimization problem assumes the form $f(\cdot)=\sum_{i=1}^Nc_ik(\cdot,\Bx_i)$, where $N$ is the number of training data points. However, this representation introduces a massive time cost of $\CalO(N^3)$ and a memory cost of $\CalO(N^2)$ in the training. Further, the feature space $\{k(\cdot,\Bx):\Bx\in\Real^D\}$ may not cover $\CalH_k$ from an optimal sense. To be more specific, there might be another set of features $\{g_i\}_{i=1}^M$ with $M\ll N$ such that  $\{k(\cdot,\Bx):\Bx\in\BX\}\subset\{g_i\}_{i=1}^M$ where $\BX \in \Real^{N\times D}$ is the input data.

To address the aforementioned problem, \cite{rahimi2008random} propose a random approximation of $k(\Bx,\Bx')$
\begin{equation}
    \label{eq:RandomFeature}
    k(\Bx,\Bx')\approx z^\mathsf{T}(\Bx)z(\Bx')
\end{equation}
where $z(\Bx)=[\zeta_1(\Bx),\ldots,\zeta_M(\Bx)]$ is a random vector. This decomposes the feature $k(\cdot,\Bx)$ into a linear combination of random low-rank features $\{\zeta_i\}$ to approximate the original target function $\sum_{i=1}^Nc_ik(\cdot,\Bx_i)$ by $\sum_{i=1}^M\alpha_i\zeta_i$. This idea resolves the computational issue of the algorithm, but due to random selection of the features, the method does not offer the best candidate features for reconstructing the target function.



Furthermore, in supervised learning the goal is to find a mapping from inputs to outputs, and an optimal kernel approximation does not necessarily result in an optimal target function. The reason is simply that we require the features that best represent the underlying data model (or target function) rather than the kernel function. 

\section{Kernel Feature Selection}
In this paper, we propose an algorithm that uses a sparse representation to attain a high prediction accuracy with a low computational cost. The key ingredient is to find an expansion:
\begin{equation}
    \label{eq:expansion}
    f=\sum_{i=1}^\infty\langle f,g_i\rangle_kg_i
\end{equation}
such that features $\{g_i\}$ satisfy the following properties:
\begin{enumerate}
    \item Compact support: supt$[g_i]$ is compact.
    \item Nested support: supt$[g_i]=\bigcup_{j\in\text{I}}\text{supt}[g_j]$ for some finite set I.
    \item Orthogonality: $\langle g_i,g_j\rangle_k=\delta_{ij}$ where $\delta_{ij}$ denotes the Kronecker delta.
\end{enumerate}
Properties 1-2 ensure low time cost for the algorithm by promoting sparsity. To be more specific, given any finite set $\{g_i\}_{i=1}^M$ and any data point $\Bx$, $g_i(\Bx)=0$ for a large number of $g_i\in\{g_i\}_{i=1}^M$. Property 3 provides a better expansion of $\CalH_k$. 

In general, this problem may be intractable; however, we will prove later in Theorem \ref{thm:Feature_D} that when $k$ satisfies the following condition, then a feature set $\{\phi_i\}$ that satisfies properties 1-3 does exist:
\begin{condition}
\label{cond:kernel}
Let kernel $k$ be of the following product form:
$$k(\Bx,\Bx')=\prod_{d=1}^Dp(\min\{x_d,x_d'\})q(\max\{x_d,x_d'\})$$
where $p$ and $q$ are the independent solutions of the Sturm-Liouville problem on the interval $[a,b]$ for any $a,b\in[-\infty,\infty]$: 
$$\frac{d}{dx}\alpha(x)\frac{dy}{dx}+\beta(x)y=0,$$
and they satisfy the following boundary conditions:
\begin{align*}
    &c_{11}p'(a)+c_{12}p(a)=0\\
    &c_{21}q'(b)+c_{22}q(b)=0
\end{align*}
with $c_{ij}\geq 0$ for $i,j=1,2$ and the operator $\frac{d}{dx}\alpha(x)\frac{d}{dx}+\beta(x)$ is an elliptic operator that satisfies Lax-Milgram Theorem (see section 6 of \cite{evans10}).
\end{condition}
We provide  two commonly used kernels that satisfy condition \ref{cond:kernel}: 
\begin{align*}
    &k(\Bx,\Bx')=e^{-\omega\|\Bx-\Bx'\|_1}\\
    &k(\Bx,\Bx')=\prod_{d=1}^D[\omega\min\{x_d,x_d'\}+1].
\end{align*}
The first one is the Laplace kernel and the second one is the kernel associated to weighted Sobolev space \cite{DKS13}. Let $z_{l,i}=i2^{-l}$ for any $l,i\in\NatInt$. Then, when the dimension $D=1$, features associated to Laplace kernel satisfying properties 1-3 are as follows: 
\begin{equation}
\label{eq:LaplaceFeature}
    \phi_{l,i}(x)=\begin{cases}
    \frac{\sinh{\omega|x-z_{l,i+1}|}}{\sinh{\omega2^{-l}}}\quad \text{if}\ x\in(z_{l,i},z_{l,i+1}]\\
    \frac{\sinh{\omega|x-z_{l,i-1}|}}{\sinh{\omega2^{-l}}}\quad \text{if}\ x\in[z_{l,i-1},z_{l,i}]\\
    0\quad \quad \quad \quad \quad \quad  \ \  \text{otherwise}
    \end{cases}
\end{equation}
and features associated to the weighted Sobolev space kernel are as follows:
\begin{align*}
    \phi_{l,i}(x)=\max\left\{0,1-\frac{|x-z_{l,i}|}{2^{-l}}\right\}
\end{align*}
where $(l,i)$ is the index of features. We now start from 1-D kernel to construct a feature space that satisfies properties 1-3: 
\begin{theorem}
\label{thm:Feature1_D}
Suppose $k$ is a kernel that satisfies Condition \ref{cond:kernel}.
Let $\BZ_l=\{z_{l,i}=i2^{-l}: i=1,2^{l}-1\}$ and let $B_l=\{i=1,\cdots,2^l-1:i \ \text{is odd}\}$. We then define the following function on the interval $[z_{l,i-1},z_{l,i+1}]=[(i-1)2^{-l},(i+1)2^{-l}]$:
\begin{equation}
\label{eq: thm1_feature}
    \phi_{l,i}(x)=\begin{cases}
 &\frac{q(x)p_{l,i+1}-p(x)q_{l,i+1}}{q_{l,i}p_{l,i+1}-p_{l,i}q_{l,i+1}}\quad\text{if}\quad x\in(z_{l,i},\ z_{l,i+1}]\\
    &\frac{p(x)q_{l,i-1}-q(x)p_{l,i-1}}{p_{l,i}q_{l,i-1}-q_{l,i}p_{l,i-1}}\quad\text{if}\quad x\in[z_{l,i-1},\ z_{l,i}]\\
    &0\quad \quad \quad \quad \quad  \quad \quad \quad \   \text{otherwise}
    \end{cases}.
\end{equation}
where $p_{l,i}=p(z_{l,i})=p(i2^{-l})$ and $q_{l,i}=q(z_{l,i})=q(i2^{-l})$. Then the following feature set is an orthogonal basis of the RKHS of $k$, $\CalH_k$, that satisfies property 1-3 on the unit interval $[0,1]$:
$$\{\phi_{l,i}:l\in\NatInt,i\in B_l\}.$$
\end{theorem}
The theorem above characterizes the set of features that satisfy Condition \ref{cond:kernel} when the input is scalar. To extend the idea to $D$-dimensional space, we only need to take the tensor product form of the 1-dimensional kernel, as described by the consequent theorem:
\begin{theorem}
\label{thm:Feature_D}
Suppose $k$ is a kernel that satisfies Condition \ref{cond:kernel}. For any $\bold{l}\in\NatInt^D$, we define the Cartesian product of sets as follows:
\begin{align*}
    &\BZ_{\bold{l}}=\times_{d=1}^D\BZ_{l_d}=\{\bold{z}_{\bold{l,i}}=(z_{l_1,i_1},\cdots,z_{l_D,i_D}):z_{l_d,i_d}\in\BZ_{l_d}\}\\
    &B_{\bold{l}}=\times_{d=1}^D=\{\bold{i}\in\NatInt^D: i_d\in B_{l_d}\}.
\end{align*}
 We then define the following function on the hypercube $\times_{d=1}^D[z_{l_d,i_d-1},z_{l_d,i_d+1}]=\times_{d=1}^D[(i_d-1)2^{-l_d},(i_d+1)2^{-l_d}]$:
 \begin{equation}
 \label{eq:thm2_feature}
     \phi_{\bold{l,i}}(\Bx)=\prod_{d=1}^D\phi_{l_d,i_d}(x_d)
 \end{equation}
 where the function $\phi_{l_d,i_d}$ is defined in Theorem \ref{thm:Feature1_D}. Then the following feature set is an orthogonal basis of the RKHS of $k$, $\CalH_k$, that satisfies property 1-3 on the unit cube $[0,1]^D$:
$$\{\phi_{\bold{l,i}}:\bold{l}\in\NatInt^D,\bold{i}\in B_\bold{l}\}.$$
\end{theorem}
The proof of Theorem \ref{thm:Feature1_D} is given in the supplementary material. Theorem \ref{thm:Feature_D} can be derived from Theorem \ref{thm:Feature1_D}, because the kernel is simply the tensor product of  1-dimensional kernel in Theorem \ref{thm:Feature1_D}.

\begin{corollary}
\label{cor:kernel_expansion}
For any kernel $k$ satisfies condition \ref{cond:kernel} and let $\phi_{\bold{l,i}}$ be the function defined in Theorem \ref{thm:Feature_D}. Then we have the following expansion for $k$:
\begin{equation}
\label{eq:kernel_expansion}
    k(\Bx,\Bx')=\sum_{\bold{l}\in\NatInt^D}\sum_{\bold{i}\in B_{\bold{l}}}\frac{\phi_{\bold{l,i}}(\Bx)\phi_{\bold{l,i}}(\Bx')}{\langle\phi_{\bold{l,i}},\phi_{\bold{l,i}}\rangle_k}
\end{equation}
where $\langle\cdot,\cdot\rangle_k$ is the inner product induced by $k$.
\end{corollary}
\begin{proof}
We only need to substitute $f(\cdot)$ in equation \eqref{eq:expansion} by $k(\Bx,\cdot)$, then according to the reproducing property of $k$ we can have the result.
\end{proof}
Corollary \ref{cor:kernel_expansion} is the direct result of Theorem \ref{thm:Feature_D}. So we can have the following sparse approximation for the value $k(\Bx,\Bx')$:
$$k(\Bx,\Bx')\approx\Bz^\mathsf{T}(\Bx)\Bz(\Bx')$$
where $$\Bz(\Bx)=\left[\frac{\phi_{\bold{l,i}}(\Bx)}{||\phi_{\bold{l,i}}||_k}\right]_{(\bold{l,i})\in S}$$ for some set $S$. We will show in section 4.1 that most entries on $z(\bold{x})$ are zero. Form this perspective, the expansion \eqref{eq:kernel_expansion} is analogous to the random feature \eqref{eq:RandomFeature} except that the above $z(\Bx)$ is nonrandom. 

We now use the RKHS of the following kernel on $[0,1]$ as an example:  $$k(x,x')=\min\{x,x'\}[1-\max\{x,x'\}].$$
 The RKHS associated to $k$ is the first order Sobolev space with zero boundary conditions:
$$\CalH_k=\left\{f:\int_{0}^1[f'(s)]^2ds<\infty, f(0)=f(1)=0\right\}.$$
In this example, the feature functions given by Theorem \ref{thm:Feature1_D} coincide with a wavelet basis in $\CalH_k$.
Consider the mother wavelet given by the triangular function:
$$\phi(d)=\max\{0,1-|d|\}.$$
Then for any $l\in\NatInt$, $i=1,\cdots,2^l-1$, direct calculations show that
\begin{equation}
    \label{eq:wavelet}
     \phi_{l,i}(x)=\phi\left(\frac{x-i2^{-l}}{2^{-l}}\right).
\end{equation}
Now it is easy to verify that the features
$\{\phi_{l,i}:l\in\NatInt,i\ \text{is odd}\}$
satisfy the desired properties 1-3:
\begin{enumerate}
    \item supt$[\phi_{l,i}]=[(i-1)2^{-l},(i+1)2^{-l}]$.
    \item supt$[\phi_{l,i}]=\text{supt}[\phi_{l+1,2i-1}]\cup\text{supt}[\phi_{l+1,2i+1}].$
    \item $\int_{0}^1\phi_{l,i}'\phi_{n,j}'ds=2^{l+1}\delta_{(l,i),(n,j)}.$
\end{enumerate}
Figure \ref{fig:waveletExpansion} illustrates the compact and nested supports of these wavelet features.
\begin{figure}
\centering
\caption{Top two panels: $\bold{W}_2=\{\phi_{l,i}:l=2\}$ and $\bold{W}_3=\{\phi_{l,i}:l=3\}$; lower two panels: nested structure for the representation  of a function $f\in\CalH_k$.}
\label{fig:waveletExpansion}
$\begin{array}{cc}
\includegraphics[width=0.23\textwidth]{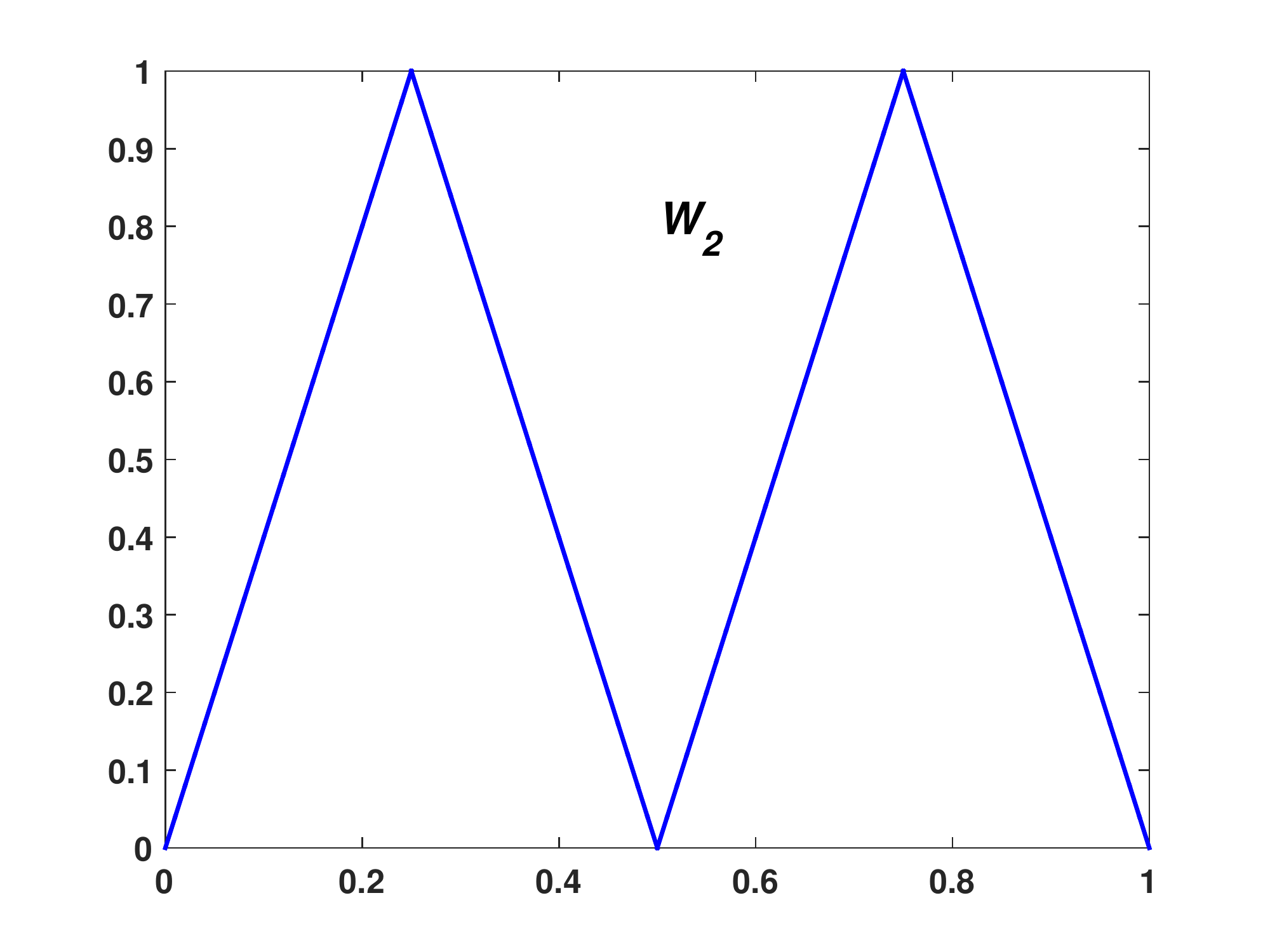}  &
\includegraphics[width=0.23\textwidth]{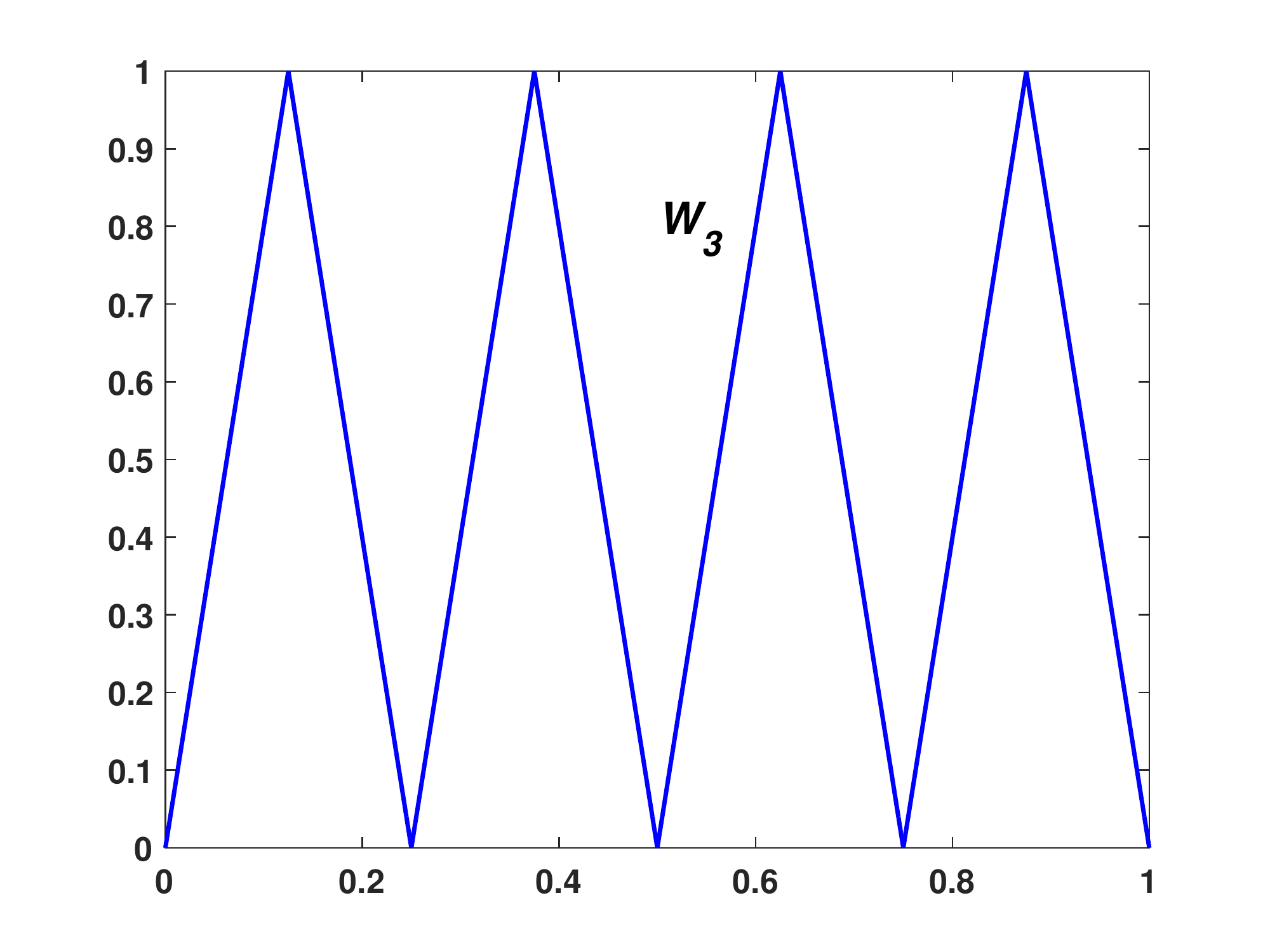} \\ 
\includegraphics[width=0.23\textwidth]{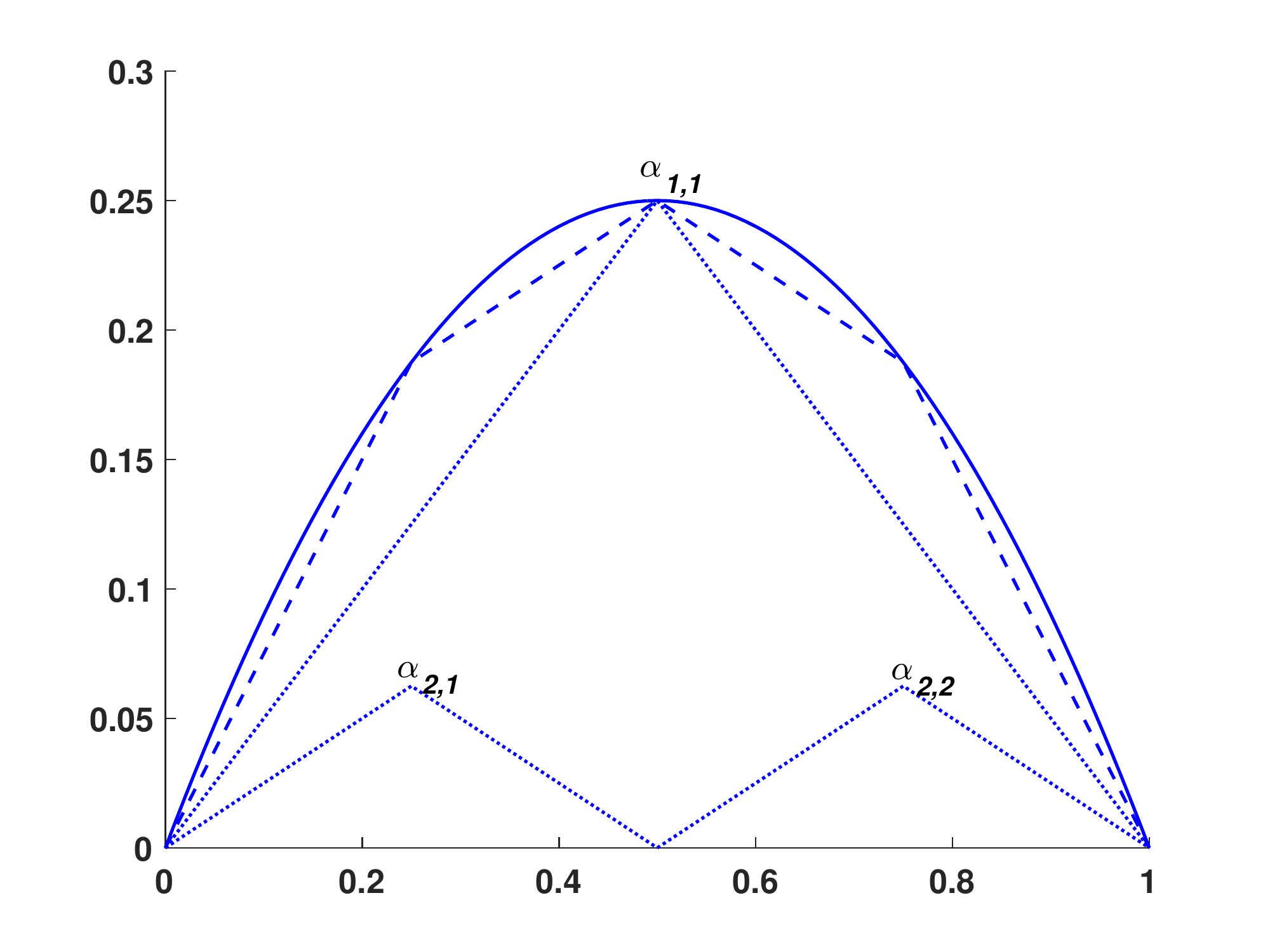} & 
\includegraphics[width=0.23\textwidth]{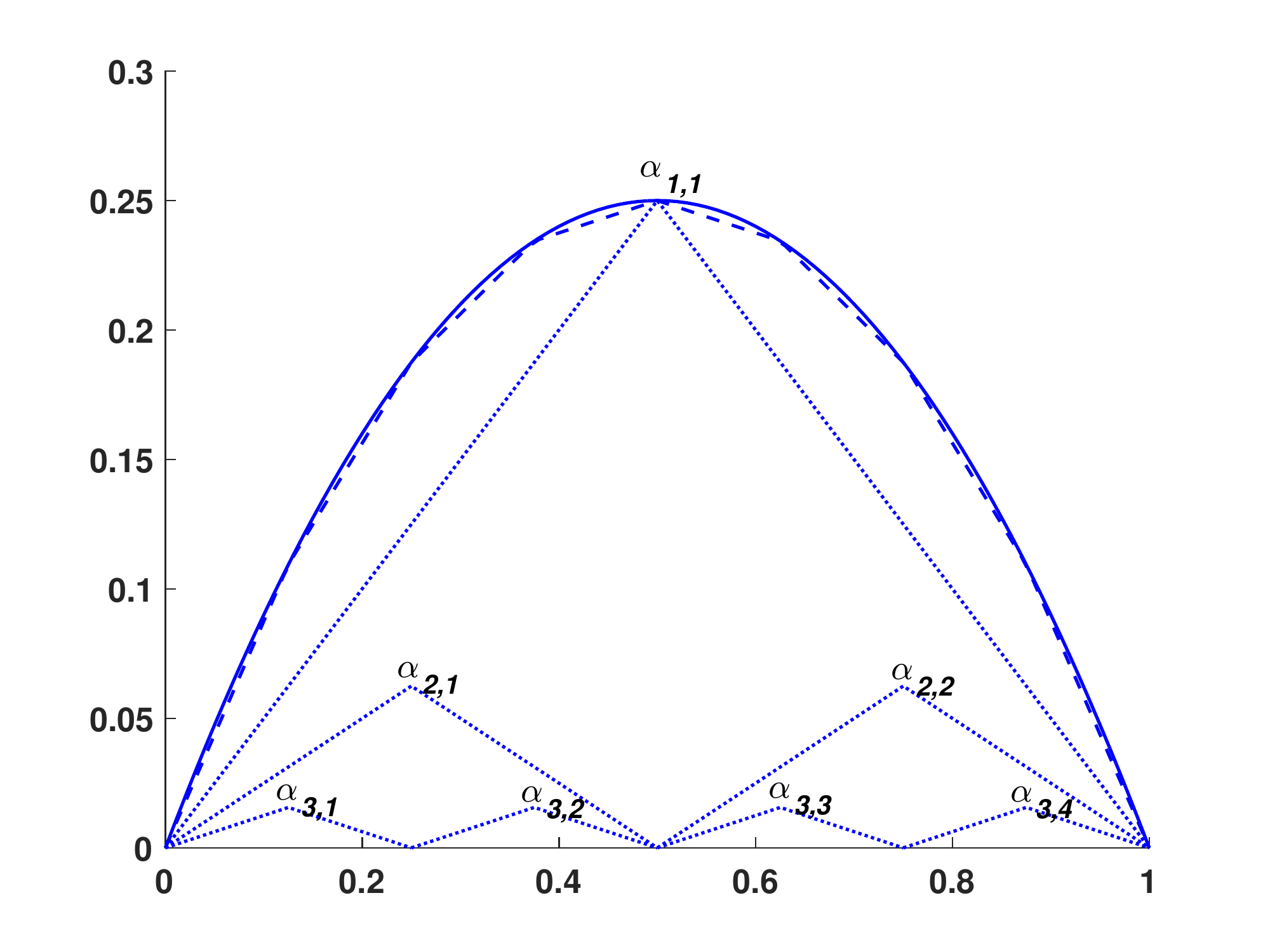}
\end{array}$
\end{figure}
The compact support properties can lead to a significant improvement in time cost. Consider the evaluation of $f(x)=\sum_{|l|\leq n} \alpha_{l,i} \phi_{l,i}(x)$. The compact support property implies that $\phi_{l,i}(x)=0$ for most $(l,i)$'s, so that the computational cost of evaluating $f(x)$ can be much lower than the total number of features. In Section \ref{Sec:alg}, we will leverage this property of the basis functions to propose an efficient algorithm for learning. This goal cannot be achieve when the basis functions are not compactly supported, such as the random features.

\begin{figure}
\centering
\caption{2-D tensor product of wavelet features with compact support $\phi_{[1,2],[1 1]}$ and $\phi_{[1,2],[1 3]}$ } 
\label{fig:tensor}
\includegraphics[width=0.53\textwidth]{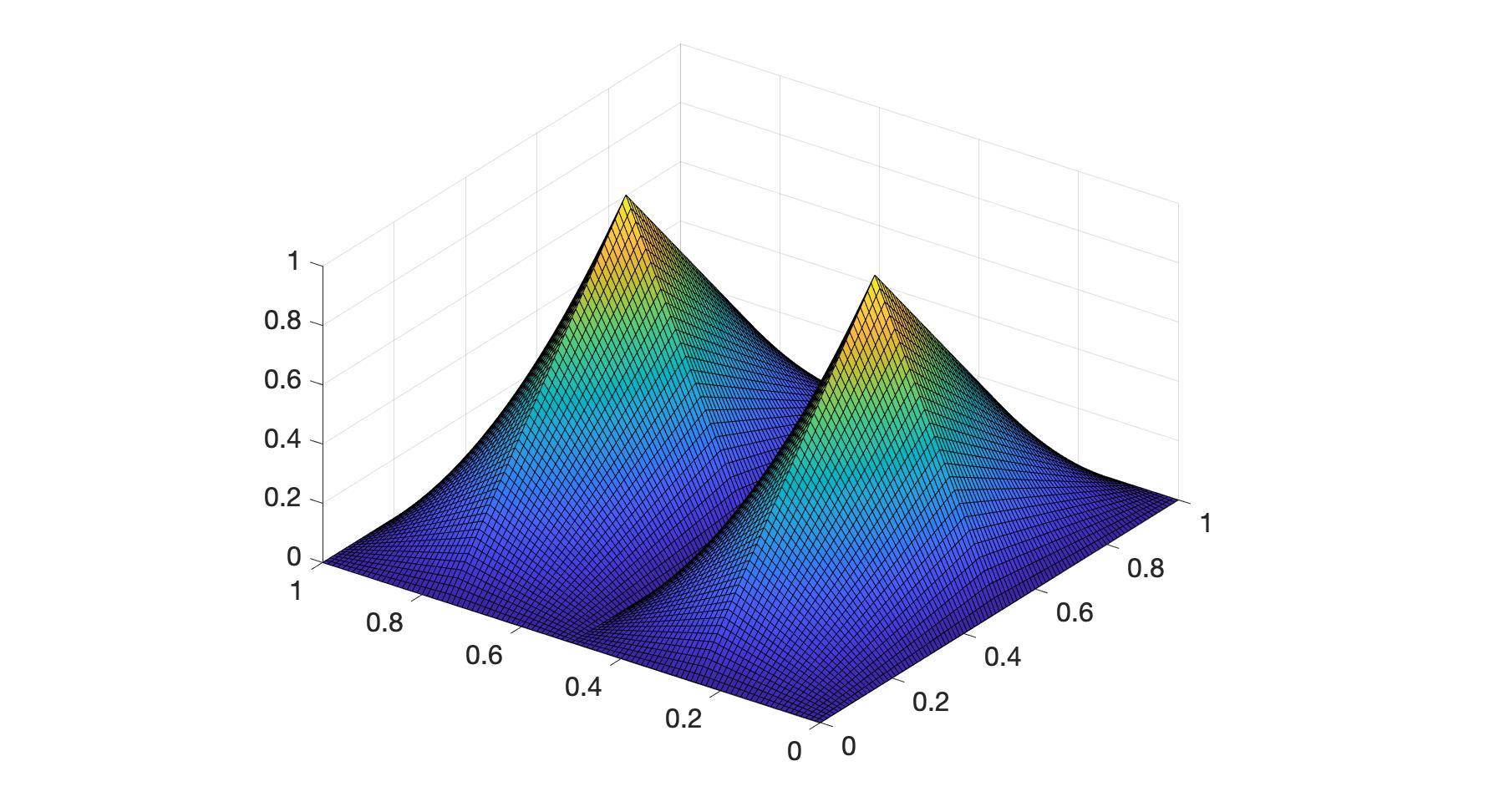}  
\end{figure}
Figure \ref{fig:tensor} shows the example of the tensor product of the wavelet feature defined in \eqref{eq:wavelet}. It is a 2-dimensional extension of the wavelet feature and according to Theorem \ref{thm:Feature_D}, the features satisfy properties 1-3 in the RKHS induced by the following kernel: 
$$k(\Bx,\Bx')=\prod_{d=1}^D\min\{x_d,x_d'\}[1-\max\{x_d,x_d'\}],$$
which is the mixed Sobolev space of first order with zero boundary condition on $[0,1]^D$.
 We refer the reader to \cite{Bungartz04} for more details on mixed order Sobolev space.

In view of Theorem \ref{thm:Feature_D}, we can now lift a data point from $\Bx\in\Real^D$ to a finite dimensional space spanned by features with compact and nested supports. As a result, the evaluation of $\Bx$ on a large number of features is zero, yielding a sparse and efficient representation. 


\section{Entropic Optimal Design}
In the previous section, we provide conditions under which we can find features with compact and nested supports. We now present an optimization criterion to select the best finite set of features with the maximum metric entropy. The intuition behind this choice is that we favor a set of features that are different from each other as much as possible, so that we can reconstruct the underlying model by a moderate amount of features.

To formulate the optimization problem, we need to introduce some notation. First we introduce the covering number of an operator between two Banach spaces. Let $\varepsilon>0$ and $\mathsf{A},\mathsf{B}$ be Banach spaces with unit balls $B_\mathsf{A}$ and $B_\mathsf{B}$, respectively. The covering number of an operator $T:\mathsf{A}\to\mathsf{B}$ is defined as
\begin{align*}
    \mathcal{N}(T,\varepsilon)=
    \inf_{n\in\NatInt}\left\{n: \exists\{b_i\in\mathsf{B}\}_{i=1}^n \ \text{s.t.}\ T(B_{\mathsf{A}})\subseteq\bigcup\limits_{i=1}^{n}(b_i+\varepsilon B_\mathsf{B})\right\}.
\end{align*}
The metric entropy of $T$ is then defined as  $\mathsf{Ent}[T,\varepsilon]:=\log \mathcal{N}(T,\varepsilon)$. Now, let $\CalH_k$ be the RKHS associated to kernel $k$ with the inner product $\langle\cdot,\cdot\rangle_k$, and let $\mathcal{P}_S$ be the projection operator from $\CalH_k$ to the following finite dimensional subspace
$$\mathcal{F}_S=\{\phi_{\bold{l,i}}: (\bold{l,i})\in S\},$$
where $\phi_{\bold{l,i}}$ is defined in Theorem \ref{thm:Feature_D} and $\text{dim}(\mathcal{P}_S)=|S|$. Our goal is to find the optimal set $S^*$ (with cardinality at most $M$), whose corresponding feature set maximizes the entropy. This is equivalent to solving the following optimization problem:
\begin{align*}
    & \sup_{S} \mathsf{Ent}[\mathcal{P}_S,\varepsilon]\\
    &\text{s.t.} \quad |S|\leq M. \numberthis \label{eq:opt}
\end{align*}
Following the lines in the proof of Theorem \ref{thm:Feature_D} (in the supplementary), we can show that the features in $\mathcal{F}_S$ are mutually orthogonal with Hilbert norm:
\begin{equation}
    \label{eq:coefficient}
    ||\phi_{\bold{l,i}}||^2_{\CalH_k}=:C^{-1}_{\bold{l,i}},
\end{equation}
where $C_{\bold{l,i}}\to\infty$ as $|\bold{l}|\to \infty$ (see lemma 1 in Supplementary Material). We first multiply $\phi_{\bold{l,i}}$ by $C_{\bold{l,i}}^{\frac{1}{2}}$ to normalize the feature. For any function $f\in\CalH_k$, we then have
$$\mathcal{P}_Sf=\sum_{(\bold{l,i})\in S}C_{\bold{l,i}}\langle f,\phi_{\bold{l,i}}\rangle_k\phi_{\bold{l,i}}.$$
As a result, the entropic optimization problem \eqref{eq:opt} is equivalent to searching an $M$-dimensional Euclidean space with the largest unit ball, which can be characterized as follows
\begin{align*}
    &\max_S\sum_{({\bold{l},\bold{i}})\in S}C_{\bold{l,i}}\\
    &\quad \text{s.t.} |S|\leq M.
\end{align*}
This optimization problem is called the Knapsack problem and, in general, is NP-hard \cite{KnapsackProblem}. However, for some specific values of $M$, closed form solutions exist. Consider the Laplace kernel here as an example. For Laplace kernel $k(\Bx,\Bx')=e^{-\omega\|\Bx-\By\|_1}$, from direct calculation, the constant is:
$$C_{\bold{l,i}}=\prod_{d=1}^D\sinh(\omega2^{-l_d}).$$
In this case, $C_{\bold{l}}=C_{\bold{l,i}}$ is independent of $\bold{i}$ and for any $|\bold{l}|<|\bold{l}'|$, the value $C_\bold{l}>C_\bold{l'}$. Therefore, we can derive that when $M=|\{\bold{l}:|\bold{l}|<n\}|$ for some $n$, the optimal set $S^*_n$ is
\begin{equation}
    \label{eq:optSet}
    S^*_n=\{({\bold{l},\bold{i}}):|\bold{l}|\leq n, \bold{i}\in B_{\bold{l}}\}
\end{equation}
because for any $C_{\bold{l}}\in S_n^*$ and any $C_{\bold{l}'}\not\in S_n^*$, $C_{\bold{l}}>C_{\bold{l}'}$.
It turns out the set $S^*_n$ is equivalent to the Sparse Grid design \cite{Bungartz04}.

\subsection{Algorithm: Entropic Optimal Features}\label{Sec:alg}
With the aforementioned theorems, we can now describe the algorithm to compute the regression function associated to  a kernel that satisfies Condition \ref{cond:kernel}. Suppose the set $S^*_n$ given by equation \eqref{eq:optSet} is the index set associated to the feature functions that maximizes the entropy optimization problem \eqref{eq:opt}. So given a specific input $\Bx$, we aim to compute the vector
$$z(\Bx)=[C_{\bold{l,i}}\phi_{\bold{l,i}}(\Bx)]_{(\bold{l,i})\in S^*_n}=: [z_{\bold{l,i}}(\Bx)]_{(\bold{l,i})\in S^*_n}$$
where $C_{\bold{l,i}}$ is the coeffecient defined in \eqref{eq:coefficient}, $z(\Bx)$ is the approximation that satisfies
$$k(\Bx,\Bx')\approx z(\Bx)^\mathsf{T}z(\Bx')$$
in Corollary \ref{cor:kernel_expansion} with $\phi_{\bold{l,i}}$  the feature function defined in equation \eqref{eq:thm2_feature}. We call $z(\Bx)$ the entropic optimal feature (EOF).

According to properties 1-3, the supports of $\{\phi_{\bold{l,i}}: (\bold{l,i})\in S^*_n\}$ are either disjoint or nested. Therefore, only a small amount of entries on $z(\Bx)$ are non-zero. To be more specific, given any $\bold{l}\in\NatInt^D$ and input $\Bx$, the supports of $\{\phi_{\bold{l,i}}: \bold{i}\in B_{\bold{l}}\}$ are disjoint so we can immediately compute the unique non-zero  entry $z_{\bold{l,i}}(\Bx)$. Algorithm \ref{alg:SPG} shows how to explicitly compute the EOF $z(\Bx)$ at a data point $\Bx$. Note that $\lceil\cdot \rceil$,$\lfloor\cdot \rfloor$ denote the ceiling and floor operations, respectively.
\begin{algorithm}
   \caption{Entropic Optimal Features (EOF)}
   \label{alg:SPG}
\begin{algorithmic}
   \STATE {\bfseries Input:} point $\Bx$, $S^*_n$
   \STATE Initialize $z(\Bx)=[z_{\bold{l,i}}(\Bx)]_{(\bold{l,i})\in S^*_n}=0$
   \WHILE{$|\bold{l}|\leq n+D-1$}
   \FOR{$d=1$ {\bfseries to} $D$}
   \STATE \begin{equation*}
       i_d=\begin{cases}
        \lceil\frac{x_d}{2^{-l_d}}\rceil\ \text{if} \ \lceil\frac{x_d}{2^{-l_d}}\rceil\ \ \text{is odd}\\
        \lfloor\frac{x_d}{2^{-l_d}}\rfloor\ \text{if} \ \lfloor\frac{x_d}{2^{-l_d}}\rfloor\ \ \text{is odd}
       \end{cases}
   \end{equation*}
   \ENDFOR
   \STATE  $z_{\bold{l},\bold{i}}(\Bx)=C_{\bold{l,i}}\phi_{\bold{l},\bold{i}}(\Bx)$
   \ENDWHILE
\end{algorithmic}
\end{algorithm}\\
 The dimension of the vector $z(\Bx)$ given $n$ levels is $\CalO(2^nn^{D-1})$ \cite{Bungartz04}. The number of non-zero elements for $z(\Bx)$ after running Algorithm \ref{alg:SPG} is: 
\begin{align*}
    \sum_{|\bold{l}|\leq n+D-1}1&=\sum_{i=D}^{n+D-1}\sum_{|\bold{l}|=i}1\\
    &=\sum_{i=D}^{n+D-1} {i-1\choose D-1}\\
    &={n+D-1\choose D}=\CalO(n^D),
\end{align*}
which means fraction of non-zeros to the whole vector in $z(\Bx)$ grows with $\CalO(\frac{n}{2^n})$ as a function of level $n$. 

{\bf Time Complexity of EOF in Regression:} Based on above, if we fix $M$ as the size of $z(\Bx)$, the number of non-zero entries on $z(\Bx)$ is $\CalO(\log^D M)$. Since we evaluate $z(\Bx)$ for each training data, the feature matrix has $\CalO(n\log^D M)$ non-zero elements, resulting in a training cost of $\CalO(N\log^{2D}M)$, which is smaller than $\CalO(NM^2)$ of random features \cite{rahimi2008random}, especially when $D$ is moderate.

\section{Generalization Bound}
In this section, we present the generalization bound for EOF when it is used in supervised learning. Let us define the approximated target function as 
$$\hat{f}:=\operatorname*{argmin}_{f\in \mathcal{F}_M}\frac{1}{N}\sum_{j=1}^N L(y_i,f(\Bx_i))+\lambda \|f\|^2_k, $$
given independent and identically distributed samples $\{(\Bx_i,y_i)\}_{i=1}^N$,
where $\mathcal{F}_M$ denotes the space spanned by the first $M$ EOFs; $L$ is a loss function; and $\lambda$ is a tuning parameter that may depend on $n$. We denote by $R(f):=\mathbb{E}_{\Bx,y}[L(y,f(\Bx))]$ the true risk. The goal is to bound the generalization error $R(\hat{f})-\inf_{f\in \mathcal{H}_k}R(f)$.

We use the following assumptions to establish the bound:
\begin{Assumption}\label{A:1}
    There exists $f_0\in \mathcal{H}_k$ so that $\inf_{f\in\mathcal{H}_k} R(f)=R(f_0)$.
\end{Assumption}
\begin{Assumption}\label{A:2}
The function $m_y(\cdot):=L(y,\cdot)$ is twice differentiable for all $y$. Furthermore, $m_y(\cdot)$ is strongly convex. 
\end{Assumption}
\begin{Assumption}\label{A:3}
    The density function of input $\Bx$ is uniformly bounded away from infinity. The outputs are uniformly bounded. 
\end{Assumption}
Assumption \ref{A:1} allows infimum to be achieved in the RKHS. This is not ensured automatically since we deal with a potentially infinite-dimensional RKHS $\mathcal{H}_k$, that is possibly universal (see Remark 2 of \cite{rudi2016generalization}).
Assumption \ref{A:2} is true for common loss functions including least squares for regression ($m_y(y')=(y-y')^2$) and logistic regression for classification ($m_y(y')=\log[1+\exp(-yy')]$). The bounded output constraint of Assumption \ref{A:3} is also common in supervised learning.  

The generalization bound is given by the following theorem.

\begin{theorem}\label{th:gb}
Suppose Assumptions \ref{A:1}-\ref{A:3} are fulfilled. If the tuning parameter is choosing to have $\lambda\sim N^{-1/2}$, then
$$R(\hat{f})-\inf_{f}R(f)\leq \CalO_p(N^{-1/2})+C M^{-2}\log^{4D-4}M, $$
for some $C>0$. The constants may depend on $\|f_0\|_k$.
\end{theorem}
The theorem above shows that with $\CalO(N^{\frac{1}{4}})$ EOFs, the optimal statistical accuracy $\CalO(1/\sqrt{N})$ is achieved up to logarithmic factors. Compared to random features for kernel approximation, this result improves the generalization bound. For random features, the number of required features to achieve the optimal rate is $\CalO(\sqrt{N})$ in the case of ridge regression \cite{rudi2016generalization}.

\section{Related Literature}\label{sec:lit}
We provide related works for kernel approximation from different perspectives:

{\bf Random Features (Randomized Kernel Approximation):} Randomized features was introduced as an elegant approach for Monte Carlo approximation of shift-invariant kernels \cite{rahimi2008random}, and it was later extended for Quasi Monte Carlo approximation \cite{yang2014quasi}. Several methods consider improving the time cost of random features, decreasing it by a linear factor of the input dimension (see e.g., Fast-food \cite{le2013fastfood,yang2015carte}). Quadrature-based random features are also shown to boost kernel approximation \cite{munkhoeva2018quadrature}. The generalization properties of random features have been studied for $\ell_1$-regularized risk minimization \cite{yen2014sparse} and ridge regression \cite{rudi2016generalization}, improving the initial generalization bound of \cite{rahimi2009weighted}. \cite{felix2016orthogonal} develop orthogonal random features (ORF) to boost the variance of kernel approximation. ORF is shown to provide optimal kernel estimator in terms of mean-squared error \cite{choromanski2018geometry}. A number of recent works have considered data-dependent sampling of random features to improve kernel approximation. Examples consist of \cite{yu2015compact} on compact nonlinear feature maps, \cite{yang2015carte,oliva2016bayesian} on approximation of shift-invariant/translation-invariant kernels, and \cite{agrawal2019data} on data-dependent approximation using greedy approaches (e.g., Frank-Wolfe). Furthermore, data-dependent sampling has been used to improve generalization in supervised learning \cite{sinha2016learning,shahrampour2018data} through target kernel alignment.

{\bf Deterministic Kernel Approximation:} The studies on finding low-rank surrogates for kernels date back two decades \cite{smola2000sparse,fine2001efficient}. As an example, the celebrated Nystr{\"o}m method \cite{williams2001using,drineas2005nystrom} samples a subset of training data for approximating a low-rank kernel matrix. The Nystr{\"o}m method has been further improved in \cite{zhang2008improved} and more recently used for approximation of indefinite kernels \cite{oglic2019scalable}. Explicit feature maps have also proved to provide efficient kernel approximation. The works of \cite{yang2004efficient,xu2006explicit,cotter2011explicit} have proposed low-dimensional Taylor expansions of Gaussian kernel for improving the time cost of learning. \cite{vedaldi2012efficient} further study explicit feature maps for additive homogeneous kernels.

{\bf Sparse Approximation Using Greedy Methods:} Sparse approximation literature has mostly focused on greedy methods. \cite{vincent2002kernel} have developed a matching pursuit algorithm where kernels are the dictionary elements. The work of \cite{nair2002some} focuses on sparse regression and classification models using Mercer kernels, and \cite{sindhwani2011non} considers sparse regression with multiple kernels. Classical matching pursuit was developed for regression, but further extensions to logistic regression \cite{lozano2011group} and smooth loss functions \cite{locatello2017unified} have also been studied. \cite{oglic2016greedy} propose a greedy reconstruction technique for regression by empirically fitting squared error residuals. \cite{shahrampour2018learning} also use greedy methods for sparse approximation using multiple kernels. 

{\it Our approach is radically different from the prior work in the sense that we characterize a set of features that maximize the entropy. Our feature construction and entropy optimization techniques are novel and have not been explored in the kernel approximation literature.}

\section{Numerical Experiments}
\textbf{Benchmark Algorithm:} We now compare \textbf{EOF} with the following random-feature benchmark algorithms on several datasets from the UCI Machine Learning Repository:

\textbf{1) RKS} \cite{rahimi2009weighted} with approximated Laplace kernel feature $z(\Bx)= \frac{1}{\sqrt{M}}[\cos(\Bx^\mathsf{T}\gammab_m + b_m)]_{m=1}^M$, where $\{\gammab_m\}^M_{m=1}$ are sampled from a Cauchy distribution multiplied by $\sigma$, and $\{b_m\}_{m=1}^M$ are sampled from the uniform distribution on $[0,2\pi]$.\\
     \textbf{2) ORF} \cite{felix2016orthogonal} with approximated Gaussian kernel feature $z(\Bx)= \frac{1}{\sqrt{M}}[\cos(\Bx^\mathsf{T}\gammab_m + b_m)]_{m=1}^M$, with $[\gammab_1\  \gammab_2\cdots\gammab_m]=\sigma \bold{S}\bold{Q}$ where $\bold{S}$ is a diagonal matrix, with diagonal entries sampled i.i.d. from the $\chi$-distribution with $d$ degrees and $\bold{Q}$ is the orthogonal matrix obtained from the QR decomposition of a matrix $\bold{G}$ with normally distributed entries. Note that ORF approximates a Gaussian kernel.\\
     \textbf{3) LKRF} \cite{sinha2016learning} with approximated Laplace kernel feature $z(\Bx)= \frac{1}{\sqrt{M}}[\cos(\Bx^\mathsf{T}\gammab_m + b_m)]_{m=1}^M$, with first
a larger number $M_0$ random features are sampled and then re-weighted by solving a kernel alignment
optimization. The top $M$ random features would be used in the training.\\
      \textbf{4) EERF} \cite{shahrampour2018data}, with approximated Laplace kernel feature $z(\Bx)= \frac{1}{\sqrt{M}}[\cos(\Bx^\mathsf{T}\gammab_m + b_m)]_{m=1}^M,$  where first
a larger number $M_0$ random features are sampled  and then re-weighted according to a score function. The top $M$ random features would appear in the training.
   
\textbf{Experiment Setup:} We also use approximated Laplace kernel feature $z(\Bx)=[\phi_{\bold{l,i}}(\Bx)]_{(\bold{l,i})\in S^*_n}$ where $\phi_{\bold{l,i}}=\prod_{d=1}^D\phi_{l_d,i_d}$ with $\phi_{l_d,i_d}$ defined as equation \eqref{eq:LaplaceFeature}.  To determine the value of $\sigma$ used in \textbf{RKS}, \textbf{EERF}, \textbf{LKRF} and  \textbf{ORF} we choose the value of $\sigma^{-1}$ for each dataset to be the mean distance of
the $50^{\text{th}}$ $\ell_2$ nearest neighbor \cite{felix2016orthogonal}. We then calculate the corresponding $\omega$ for $\textbf{EOF}$ associated to $\sigma$.  
The number of features in \textbf{EOF} is a function of dimension $D$ and level $n$, so it is not possible to calculate them for any $M$. To resolve this issue, for any given $M$, we select the set $S^*_n$ defined in \eqref{eq:optSet} that satisfies
\begin{equation*}
 \big|S^*_{n-1}\big|<M\leq \big|S^{*}_{n}\big|  
\end{equation*}
and randomly select $M$ pairs of $(\bold{l,i})\in S^*_n$ to have a random set $S_M$. We then use the following  approximated feature:
$$z_M(\bold{x}):=[\phi_{\bold{l,i}}(\Bx)]_{(\bold{l,i})\in S_M}.$$ This is equivalent to randomly select $M$ rows from the feature $z(\Bx)=[\phi_{\bold{l,i}}(\Bx)]_{(\bold{l,i})\in S^*_n}$.

We let $M_0=10M$ for \textbf{LKRF} and \textbf{EERF}, then for any $M$, we compare the performance of different algorithms.

\textbf{Datasets:} In Table \ref{sample-table}, we report the number of training samples $N_{\text{train}}$ and test samples $N_\text{test}$ used
for each dataset. For the MNIST data set, we  map the original $784-$dimensional data to a $32-$dimensional space using an auto-encoder. If the training and test samples are not provided separately for a dataset, we split it
randomly. We standardize the data as follows: we scale each input to the unit interval $[0,1]$  and the responses in regression to be inside $[-1, 1]$.
\begin{table}
\caption{Input dimension, number of training samples, and number of test samples are denoted by $D$, $N_\text{train}$, and
$N_\text{test}$, respectively}
\label{sample-table}
\vskip 0.15in
\begin{center}
\begin{small}
\begin{sc}

{\scriptsize
\scalebox{1.24}{
\begin{tabular}{lcccc}
\toprule
Data set &Task &D& $N_\text{train}$ & $N_\text{test}$ \\
\midrule
MNIST &Classification & 32& 20000& 10000 \\
Electrical Grids Stability & Classification   & 13& 7000& 3000 \\
Superconductivity & Regression& 81&15000&6263 \\
Energy Efficiency & Regression&8&512&256\\
\bottomrule
\end{tabular}
}}
\end{sc}
\end{small}
\end{center}
\vskip -0.1in
\end{table}

\begin{figure*}[t!]
\caption{Comparison of the test error of EOF (this work) versus benchmark algorithms inclduing RKS, EERF, LKRF and ORF.}
\label{fig:performance}
\includegraphics[width=1\textwidth]{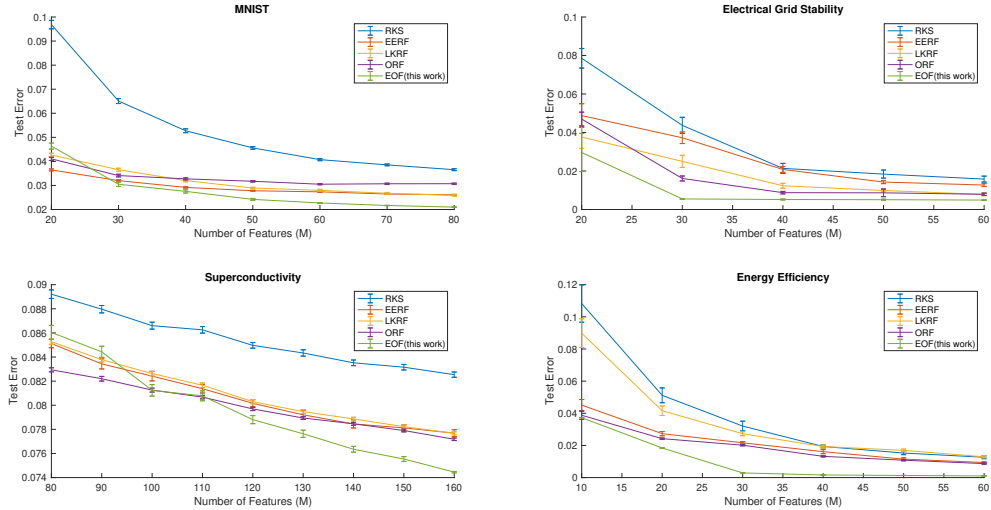}  
\end{figure*}

\textbf{Comparison:} For a fixed number of features, we perform 50 simulation runs for each algorithm on each data set. We then report the average test error (with standard errors) in
Fig. \ref{fig:performance} where the plot line is the mean error of an algorithm and the error bar reflects the standard deviation of the error. Throughout our experiments, we can see that \textbf{EOF} consistently improves the test error compared
to other randomized-feature algorithms. 
This is specifically visible when the gap between $S_M$ and $S_n^*$ becomes very small and, due to the optimality of $S^*_n$, \textbf{EOF} outperforms any random feature algorithm.

\begin{table*}[t!]
\centering
\caption{Time and Space Complexity Comparison}
\label{tab:Time&Space}
\begin{tabular}{ |p{1cm}||p{0.6cm}|p{0.6cm}|p{0.6cm}|p{1.3cm}|  }
 \hline
 \multicolumn{5}{|c|}{MNIST} \\
 \hline
 Method & $M$& $M_0$& $T_{\rm train}$ & nnz($F$) \\
 \hline
 RKS   &    80  &  &  1.64 &$1.6\times10^6$ \\
 EERF&  80   &  800  &4.43 &$1.6\times10^6$ \\
 LKRF & 80  & 800 & 3.07 &$1.6\times10^6$ \\
 ORF&  80   &    &1.21 &$1.6\times10^6$ \\
 EOF & 80 &   2048& 2.45 & $2.5\times10^5$\\
 \hline
 \hline
 \multicolumn{5}{|c|}{Superconductivity} \\
 \hline
Method & $M$& $M_0$& $T_{\rm train}$ & nnz($F$) \\ \hline
 RKS   &    160  & & 0.10  &$2.4\times 10^6$ \\
 EERF&  160   &  1600 &0.45 &$2.4\times 10^6$ \\
 LKRF & 160  & 1600 & 0.37 &$2.4\times 10^6$ \\
 ORF & 160  &  & 0.13 &$2.4\times 10^6$ \\
 EOF & 160 &   161& 0.14 & $1.2\times 10^6$\\
 \hline
 \end{tabular}
 \begin{tabular}{ |p{1cm}||p{0.6cm}|p{0.6cm}|p{0.6cm}|p{1.3cm}|  }
 \hline
 \multicolumn{5}{|c|}{Electrical Grids Stability} \\
 \hline
 Method & $M$& $M_0$& $T_{\rm train}$ & nnz($F$) \\
 \hline
 RKS   &    60  &  & 0.04  &$4.2\times10^5$ \\
 EERF&  60   &  600  &0.14 &$4.2\times10^5$ \\
 LKRF & 60  & 600 & 0.13 &$4.2\times10^5$ \\
 ORF & 60  &  & 0.06 &$4.2\times10^5$ \\
 EOF & 60 &    338& 0.08 & $1.3\times10^5$\\
 \hline
 \hline
 \multicolumn{5}{|c|}{Energy Efficiency} \\
 \hline
 Method & $M$& $M_0$& $T_{\rm train}$ & nnz($F$) \\
 \hline
 RKS   &    60  &  & 0.01  &$6.1\times 10^3$ \\
 EERF&  60   &  600 &0.05 &$6.1\times 10^3$ \\
 LKRF & 60  & 600 & 0.06 &$6.1\times 10^3$ \\
 ORF & 60  &  & 0.02 &$6.1\times 10^3$ \\
 EOF & 60 &   128& 0.03 & $1.0\times 10^3$\\
 \hline
\end{tabular}
\end{table*}

In Table \ref{tab:Time&Space}, we also compare the time complexity and space complexity. We define the feature matrix
$$F:=[z(\Bx_i)]_{i=1}^N,$$
which is an $M\times N$ matrix with $M$ the number of features and $N$ the number of data. Due to the sparse structure of $\textbf{EOF}$, we can also see that the number of non-zero entries of the $F$ associated to \textbf{EOF} is smaller than other methods. When both the dimension $D$ and the size of data $N$ are large, the sparsity of \textbf{EOF} becomes more obvious as shown in the case of MNIST. The time cost of running \textbf{EOF} is also quite impressive. It is consistently better than \textbf{EERF} and \textbf{LKRF} and comparable and slightly slower than \textbf{RKS}. In fact, the major time for \textbf{EOF} is spent on feature matrix construction. For random features, due to high efficiency of matrix operations in Matlab, feature construction is fast. However, for \textbf{EOF} the feature construction via matrix operations is not possible in an efficient way. We observed that after the feature matrix construction, \textbf{EOF} is the fastest method in training. For example, if we only count the training time (excluding feature construction) as the time cost, in kernel ridge regression on the dataset Superconductivity, the comparison between \textbf{RKS} and \textbf{EOF} is as follows:

\begin{table}[h!]
\centering
\caption{Comparison on RKS and EOF in pure training excluding feature construction.}
\scalebox{0.8}{
\begin{tabular}{l c c c c c} 
 \hline
    &$M=80$ &$M=100$&$M=120$&$M=140$&  $M=160$ \\ [0.5ex] 
 \hline
\textbf{RKS}& $2\times 10^{-3}$  &$3\times 10^{-3}$  &$4\times 10^{-3}$& $5\times 10^{-3}$ & $6\times 10^{-3}$  \\
\textbf{EOF}& $2\times 10^{-3}$ &$2\times 10^{-3}$   &$2\times 10^{-3}$ & $2\times 10^{-3}$& $2\times 10^{-3}$ \\ [1ex] 
 \hline
\end{tabular}
}
\label{table:1}
\end{table}

The run time is obtained on a MacPro with a 4-core, 3.3 GHz Intel Core i5 CPU  and 8 GB of RAM (2133Mhz).

\section{Conclusion}
We provide a method to construct a set of mutually orthogonal features (with nested and small supports) and select the best $M$ of them that maximize the entropy of the associated projector. The nested and compact support of feature functions greatly reduces the time and space cost for feature matrix operations. The orthogonality and entropic optimality reduces dramatically the error of approximation. We have provided generalization error bound which indicates that only $\CalO(N^{\frac{1}{4}})$ features are needed to achieve the $\CalO(N^{-\frac{1}{2}})$ optimal accuracy. Future directions include generalizing this method to a broader class of kernels.


\appendix
\section*{Supplementary Material}
\section{Proof of Theorem \ref{thm:Feature1_D}}
The kernel function:
$$k(x,y)=p(\min\{x,y\})q(\max\{x,y\})$$
is in fact the Green's function of the Sturm-Liouville operator \cite{ODEHandbook}
$$\mathcal{L}:=\frac{d}{dx}\alpha(x)\frac{d}{dx}+\beta(x)$$
So the inner product product induced by $k$ is
$$\langle f,g\rangle_k=\int_{0}^1f\mathcal{L}gdx$$
For any $l\in\NatInt$ and $i\neq j$, the supports of $\phi_{l,i}$ and $\phi_{l,j}$ are $[(i-1)2^{-l},(i+1)2^{-l}]$ and $[(j-1)2^{-l},(j+1)2^{-l}]$respectively. This two supports are disjoint because both $i$ and $j$ are odd so $\langle \phi_{l,i},\phi_{l,j}\rangle_k=0$ if $i\neq j$.
For any $l,n\in\NatInt$ and any $i,j$, the supports supt$[\phi_{l,i}]$ and supt$[\phi_{n,j}]$ are either disjoint or nested. If they are disjoint, then $\langle\phi_{n,j} ,\phi_{n,j}\rangle_k=0$. If they are nested, , without loss of generality assume $l>n$ and $i\leq j2^{l-n}$, then because both $p$ and $q$ satisfy:
$$\mathcal{L}p=\mathcal{L}q=0$$
so
\begin{align*}
    &\quad\langle\phi_{l,i},\phi_{n,j}\rangle_k\\ &=\int_{(i-1)2^{-l}}^{(i+1)2^{-l}}\phi_{n,j}\mathcal{L}\phi_{l,i}dx\\
    &=\int_{(i-1)2^{-l}}^{(i+1)2^{-l}}\phi_{n,j}\mathcal{L}\frac{p(x)q_{l,i-1}-q(x)p_{i,l-1}}{p_{l,i}q_{l,i-1}-q_{l,i}p_{l,i-1}}dx\\
    &=0.
\end{align*}
As a result, we have 
$$ \langle\phi_{l,i},\phi_{n,j}\rangle_k=\lambda_{l,i}\delta_{(l,i),(n,j)} $$
where $\lambda_{l,i}$ is a function of $l$ and $i$.

\section{Proof of Theorem \ref{th:gb}}

We need the following lemmas.

\begin{lemma}\label{lemma:1}
Denote $f_M=\operatorname*{argmin}_{f\in\mathcal{F}_M} \|f_0-f\|_k$. Then we have 
$$R(f_M)-R(f_0)\leq C M^{-2}\log^{4D-4} M \|f_0\|^2_k, $$
for some constant $C$.
\end{lemma}
\begin{proof}
According to Assumption \ref{A:2}, we can see that 
$$R(f_M)-R(f_0)=\mathbb{E}[m''_y(\bold{u}^*)(f_M(x)-f_0(x))^2]$$
In view of Assumption \ref{A:3}, we only need to prove
$$\|f_M-f_0\|_{L^2}^2=C M^{-2}\log^{4D-4} M \|f_0\|_k^2$$
for any $f_0\in\CalH_k$ we then can finish the proof. Let $M=|\{(\bold{l,i}):|\bold{l}|\leq n,\bold{i}\in B_{\bold{l}}\}|$. According to theorem \ref{thm:Feature_D}, we have the following expansion:
\begin{align*}
    &\quad \|f_M-f_0\|_{L^2}\\
    &=\|\sum_{|\bold{l}|>n}\sum_{\bold{i}\in B_{\bold{l}}}\langle f_0,\frac{\phi_{\bold{l,i}}}{\|\phi_{\bold{l,i}}\|_k}\rangle_k\frac{\phi_{\bold{l,i}}(\cdot)}{\|\phi_{\bold{l,i}}\|_k}\|_{L^2}\\
    &= \|\sum_{|\bold{l}|>n}\sum_{\bold{i}\in B_{\bold{i}}}\int_{\bold{S}_{\bold{l,i}}}f_0(\bold{s})\mathcal{L}\phi_{\bold{l,i}}(\bold{s})d\bold{s}\frac{\phi_{\bold{l,i}}(\cdot)}{\|\phi_{\bold{l,i}}\|_k^2}\|_{L^2}.
\end{align*}
where $\bold{S}_{\bold{l,i}}$ is the support of $\phi_{\bold{l.i}}$. We let
$$v(\cdot)_{\bold{l}}:=\sum_{\bold{i}\in B_{\bold{i}}}\int_{\bold{S}_{\bold{l,i}}}f_0(\bold{s})\mathcal{L}\phi_{\bold{l,i}}(\bold{s})d\bold{s}\frac{\phi_{\bold{l,i}}(\cdot)}{\|\phi_{\bold{l,i}}\|_k^2}.$$
Our first goal is to estimate $v_{\bold{l}}$. From theorem 2 of \cite{ding2018scalable} or direct calculation based on the property of Green's function, we can see that for any $f\in\CalH_k$:
\begin{equation*}
    \int_{\bold{S}_{\bold{l,i}}}f(\bold{s})\mathcal{L}\phi_{\bold{l,i}}(\bold{s})d\bold{s}=[\bigotimes_{d=1}^D\Delta_{l_d,i_d}]f
\end{equation*}
where 
\begin{align*}
    &\Delta_{l_d,i_d} f:=\alpha_{l_d,i_d}f\big|_{x_d=z_{l_d,i_d}}\\
    &\quad\quad\quad\ \ -\beta_{l_d,i_d-1}f\big|_{x_d=z_{l_d,i_d-1}}-\beta_{l_d,i_d+1}f\big|_{x_d=z_{l_d,i_d}}\\
    &\alpha_{l,i}=\frac{p_{l,i+1}q_{l,i-1}-p_{l,i-1}q_{l,i+1}}{[p_{l,i+1}q_{l,i}-p_{l,i}q_{l,i+1}][p_{l,i1}q_{l,i-1}-p_{l,i-1}q_{l,i}]}\\
    &\beta_{l,i}=\frac{1}{p_{l,i+1}q_{l,i}-p_{l,i}q_{l,i+1}}
\end{align*}
and $\bigotimes$ denotes the tensor product of the $\Delta_{l,i}$ operators. Since bouth $q$ and $p$ are the solution of the SL-equation, therefore, $p,q$ are twice differentiable. We have 
\begin{align*}
    &\quad\frac{1}{p_{l,i+1}q_{l,i}-p_{l,i}q_{l,i+1}}\\
    &=\frac{2^l}{[p_{l,i+1}q_{l,i}-p_{l,i}q_{l,i}]/2^{-l}-[p_{l,i}q_{l,i+1}-p_{l,i}q_{l,i}]/2^{-l}}\\
    &\sim\frac{2^l}{p'_{l,i}q_{l,i}-p_{l,i}q'_{l,i}}
\end{align*}
we notice that $p'_{l,i}q_{l,i}-p_{l,i}q'_{l,i}$ is the Wronskian of the SL-operator, which is bounded away from 0. Therefore, $\Delta_{l_d,i_d}$ acting on $f$  has the following approximation: 
\begin{align*}
    \Delta_{l_d,i_d}f&\sim \frac{[2f\big|_{x_d=z_{l_d,i_d}}-f\big|_{x_d=z_{l_d,i_d-1}}-f\big|_{x_d=z_{l_d,i_d+1}}]}{2^{-l}}\\
    &\leq C \max_{j=1,-1}\{\frac{|f\big|_{x_d=z_{l_d,i_d+j}}-f\big|_{x_d=z_{l_d,i_d}}|}{2^{-l}}\}.
\end{align*}
As a result, $\bigotimes_{d=1}^D\Delta_{l_d,i_d}$ acting on $f$  has the following approximation:
\begin{align*}
    &\quad\bigotimes_{d=1}^D\Delta_{l_d,i_d}f\\
    &\leq C \prod_{d=1}^D\max_{j=1,-1}\{\frac{|f\big|_{x_d=z_{l_d,i_d+j}}-f\big|_{x_d=z_{l_d,i_d}}|}{2^{-l}}\}.
\end{align*}
From the same reasoning, we can see that
$$\|\phi_{\bold{l,i}}\|_k^2=\prod_{d=1}^D\alpha_{l_d,i_d}\sim 2^{|\bold{l}|}.$$
We also Taylor expand $\phi_{l_d,i_d}$ for each $1\leq d\leq D$ up to second order and from direct calculation, we can have
\begin{align*}
    \phi_{l_d,i_d}(x)\sim \max\{0,1-\frac{|x-z_{l_d,i_d}|}{2^{-l_d}}\}+\CalO(2^{-l_d}).
\end{align*}
This gives us the approximation up to second order:
\begin{align*}
    &\quad \|\phi_{\bold{l,i}}\|^2_{L_2}\\
    &=\int_{\bold{S}_\bold{l,i}}\prod_{d=1}^D\phi^2_{l_d,i_d}(s_d)d\bold{s}\\
    &\sim \int_{\bold{S}_\bold{l,i}}\prod_{d=1}^D[\max\{0,1-\frac{|s-z_{l_d,i_d}|}{2^{-l_d}}\}]^2\bold{s}\\
    &=\big(\frac{2}{3}\big)^D 2^{-|\bold{l}|}=\big(\frac{1}{3}\big)^D\text{Vol}(\bold{S}_\bold{l,i}).
\end{align*}
Therefore, we can have the following estimate for $v_{\bold{l}}$:
\begin{align*}
    \|v_{\bold{l}}\|_{L^2}&=\|\sum_{\bold{i}\in B_{\bold{i}}}\int_{\bold{S}_{\bold{l,i}}}f_0(\bold{s})\mathcal{L}\phi_{\bold{l,i}}(\bold{s})d\bold{s}\frac{\phi_{\bold{l,i}}(\cdot)}{\|\phi_{\bold{l,i}}\|_k^2}\|_{L^2}\\
    &\leq \Big|2^{-2|\bold{l}|} C\sum_{\bold{i}\in B_{\bold{i}}}[\bigotimes_{d=1}^D\Delta_{l_d,i_d}f]^2\text{Vol}(\bold{S}_\bold{l,i})\Big|^{\frac{1}{2}}\\
    &\sim 2^{-|\bold{l}|}\|\prod_{d=1}^D\frac{\partial}{\partial x_d}f_0\|_{L^2}\\
    &\sim 2^{-|\bold{l}|}\|f_0\|_{k}
\end{align*}
where the second line is from the fact that supports of $\{\phi_{\bold{l,i}}:\bold{i}\in B_{\bold{l}}\}$ are disjoint, the third line is from the Riemann integral approximation and the last line is from the energy estimate assumption of SL-operator (see, for instance, section 6.2.2 of \cite{evans10}). Finally, we have:
\begin{align*}
    \|f_0-f_M\|_{L^2}&\leq \sum_{|\bold{l}|>n}\|v_{\bold{l}}\|_{L^2}\\
    &\sim \|f_0\|_k\sum_{|\bold{l}|>n}2^{-|\bold{l}|}\\
    &=\|f_0\|_k\sum_{i>n}2^{-i}\sum_{|\bold{l}|=i}1\\
    &=\|f_0\|_k\sum_{n>i}2^{-i}{i-1\choose d-1}\\
    &\sim\|f_0\|_k2^{-n}n ^{D-1}
\end{align*}
where the identity of the last line can be verified in \cite{Ding19}. From  \cite{Bungartz04} we also have 
$$M=\CalO(2^nn^{D-1})$$
we can substitute this identity to the previous equation to have the final result.
\end{proof}
The $(\epsilon,L_\infty)$-covering number of a function space $\mathcal{F}$, denoted as $N(\epsilon,\mathcal{F},\|\cdot\|_{L_\infty})$, is defined as the smallest number $N_0$, so that there exist centers $f_1,\ldots,f_{N_0}$, and for each $f\in\mathcal{F}$, there exists $f_i$ so that $\|f-f_i\|_{L_{\infty}}<\epsilon$.

\begin{lemma}\label{lemma:2}
The covering number of the unit ball of $\mathcal{H}_k$, denoted as $\mathcal{F}:=\{f\in\mathcal{H}_k:\|f\|_{k}\leq 1\}$, 
is bounded as follows:
$$N(\epsilon,\mathcal{F},\|\cdot\|_{L_\infty})=\CalO(\frac{1}{\varepsilon}\log ^{D-\frac{1}{2}}\frac{1}{\varepsilon})$$
\end{lemma}
\begin{proof}
When $k(\Bx,\By)=e^{-\omega\|\Bx-\By\|_1}$or $k(\Bx,\By)=\prod_{d=1}^D\min\{x_d,y_d\}$, then $\CalH_k$ is equivalent to the Sobolev space of mixed first derivative $\CalH^1_{\rm mix}([0,1]^D)$ \cite{Ding19}. According to 6.6 of \cite{HyperbolicCross}, we can immediately derive the result. When kernel $k$ is different than these two, the energy property of an SL-operator requires that
\begin{align*}
    \langle f,f\rangle_k&=\int_{[0,1]^D}f(\bold{x})[\prod_{d=1}^D\mathcal{L}]f(\bold{x})d(\bold{x})\\
    &\leq C\int_{[0,1]^D}|\prod_{d=1}^D\frac{\partial}{\partial x_d}f|^2d\Bx
\end{align*}
which means $\CalH_k$ can be embedded on $\CalH^1_{\rm mix}$. Therefore, the covering number of $\CalH_k$ must be bounded by that of $\CalH^1_{\rm mix}$.
\end{proof}

Lemma \ref{lemma:Donsker} shows the the function classes associated with the learning problem are Donsker. We refer to \cite{van1996weak} for the definition and properties of Donsker classes.

\begin{lemma}\label{lemma:Donsker}
Let $P$ be the probability measure of $(x,y)$.
The space $\mathcal{G}_R$ is $P$-Donsker for each $R>0$.
\end{lemma}

\begin{proof}
In view of Theorem 2.5.6 of \cite{van1996weak}, it suffices to prove that
$$\int_0^\infty \sqrt{\log N_{[]}(\epsilon,\mathcal{G}_R,\|\cdot\|_{L_2(P)})}d \epsilon<\infty, $$
where $N_{[]}(\epsilon,\mathcal{G}_R,\|\cdot\|_{L_2(P)})$ is the covering number with bracketing defined as follows. For function $g:\mathbb{R}^D\times\mathbb{R}\rightarrow \mathbb{R}$, its $L_2(P)$ norm is defined as $[\mathbb{E}[g(x,y)]^2]^{1/2}$. Given functions $g_L,g_U$ such that $g_L(\bold{u},v)\leq g_U(\bold{u},v)$ for each $(\bold{u},v)$, define the bracket $[g_L,g_U]$ as the set of functions $\{g:g_L(\bold{u},v)\leq g(\bold{u},v)\leq g_U(\bold{u},v)\}$. The covering number with bracketing $N_{[]}(\epsilon,\mathcal{G}_R,\|\cdot\|_{L_2(P)})$ is the smallest number $N_0$ so that there exist brackets $[g_{L,1},g_{U,1}],\ldots,[g_{L,N_0},g_{U,N_0}]$, such that $\cup_{i=1}^{N_0}[g_{L,i},g_{U,i}]\supset \mathcal{G}_R$, and $\|g_{U,i}-g_{L,i}\|_{L_2(P)}\leq \epsilon$ for all $i$.

Let $\mathcal{F}_R=\{f:\|f\|_k< R\}$.
We start with the centers $f_1,\ldots,f_{N_0}$ with $N_0=N(\epsilon,\mathcal{F}_R,\|\cdot\|_{L_\infty})=N(\epsilon/R,\mathcal{F}_1,\|\cdot\|_{L_\infty})$ so that for each $f\in \mathcal{F}_R$, there exists $f_i=:\xi(f)$ such that $\|f-f_i\|_{L_\infty}<\epsilon$. To bound the covering number with bracketing, we need to construct the associated brackets. The reproduction property implies that $\|f\|_{L_\infty}\leq \|f\|_k$. Then for any $f\in \mathcal{F}_R$, by mean value theorem,
\begin{eqnarray}
|L(y,f(\bold{x}))-L(y,\xi(f)(\bold{x}))|&\leq& \sup_{|\bold{u}|<R}\left|\frac{\partial L}{\partial u}(y,\bold{u})\right|\epsilon\nonumber
=: S(y)\epsilon.\label{bracket}
\end{eqnarray}
Now we define $g_{L,i}(\bold{u},v)=L(v,f_i(\bold{u}))-S(v)\epsilon$ and $g_{U,i}(\bold{u},v)=L(v,f_i(\bold{u}))+S(v)\epsilon$. Clearly $g_{L,i}\leq g_{U,i}$ and $$\|g_{U,i}-g_{L,i}\|_{L_2(P)}=2\epsilon [\mathbb{E}[S(y)]^2]^{1/2}, $$
which is a multiple of $\epsilon$ according to Assumptions \ref{A:2}-\ref{A:3}. Besides, (\ref{bracket}) implies that for all $f$ such that $\|f-f_i\|_{L_\infty}<\epsilon$, $L(v,f(\bold{u}))\in [g_{L,i},g_{U,i}]$. So we invoke Lemma \ref{lemma:2} to find that
$$N_{[]}(2\epsilon [\mathbb{E}[S(y)]^2]^{1/2},\|\cdot\|_{L_2(P)})=\CalO(\frac{1}{\varepsilon}\log ^{D-\frac{1}{2}}\frac{1}{\varepsilon}), $$
which implies the desired result.
\end{proof}

To bound the generalization error, we observe
\begin{align*}
    R(\hat{f})-R(f_0)&= \left\{R(\hat{f})-\frac{1}{N}\sum_{i=1}^N L(y_i,\hat{f}(\Bx_i))\right\}\\
    &+\left\{\frac{1}{N}\sum_{i=1}^N L(y_i,\hat{f}(\Bx_i))-\frac{1}{N}\sum_{i=1}^N L(y_i,f_M(\Bx_i))\right\}\\
    &+\left\{\frac{1}{N}\sum_{i=1}^N L(y_i,f_M(\Bx_i))-R(f_M)\right\}\\
    &+\left\{R(f_M)-R(f_0)\right\}=:I_1+I_2+I_3+I_4.
\end{align*}
We will bound $I_1$ and $I_3$ with a uniform error bound of empirical processes. For $I_2$, we have
$$I_2\leq \lambda\|f_M\|^2_k-\lambda \|\hat{f}\|^2_k\leq \lambda \|f_0\|^2_K=O(N^{-1/2})\|f_0\|^2_k, $$
where the first inequality follows from the optimality condition
\begin{eqnarray*}
\frac{1}{N}\sum_{i=1}^N L(y_i,\hat{f}(\Bx_i))+\lambda\|\hat{f}\|^2_k
\leq \frac{1}{N}\sum_{i=1}^N L(y_i,f_M(\Bx_i))+\lambda\|f_M\|^2_k. 
\end{eqnarray*}
The term $I_4$ is bounded by Lemma \ref{lemma:1}.

Now we turn to $I_1$ and $I_3$. To show that $I_1=O_p(N^{-1/2})$ and $I_3=O_p(N^{-1/2})$, it suffices to show that the functions $L(y,\hat{f}(\Bx))$ and $L(y,f_M(\Bx))$ fall in a Donsker class \cite{van1996weak} with probability arbitrarily close to one. For $L(y,f_M(\Bx))$, this is clearly true in view of Lemma \ref{lemma:Donsker} and the fact that $\|f_M\|_k\leq \|f_0\|_k$. Therefore, $I_3=O_p(N^{-1/2})$ For $L(y,\hat{f}(\Bx))$, it suffices to prove that $\|\hat{f}\|_k=O_p(1)$. To show this result, we start with the optimality condition
\begin{eqnarray*}
\frac{1}{N}\sum_{i=1}^N L(y_i,\hat{f}(\Bx_i))+\lambda\|\hat{f}\|^2_k \leq \frac{1}{N}\sum_{i=1}^N L(y_i,f_M(\Bx_i))+\lambda\|f_M\|^2_k. 
\end{eqnarray*}
In view of Assumption \ref{A:2}, we can write
\begin{eqnarray*}
    L(y,f(\Bx))-L(y,f_0(\Bx))
    =W\cdot(f(\Bx)-f_0(\Bx))+ m''_y(u^*)(f(\Bx)-f_0(\Bx))^2,
    \end{eqnarray*}
    where $W=m'_y(f_0(\Bx))$, and $u^*$ lies between $f(\Bx)$ and $f_0(\Bx)$. Assumptions 2 and 3 implies that the derivative and the expectation are interchangeable, so that
    $$0=\left(\mathbb{E}m_y(f_0(\BX))\right)'=\mathbb{E}m'_y(f_0(\BX))=\mathbb{E}W. $$
We then invoke Assumption \ref{A:2} to find
\begin{eqnarray}
&&\lambda \|\hat{f}\|^2_k\leq -\frac{1}{N}\sum_{i=1}^N W_i (\hat{f}(\Bx_i)-f_0(\Bx_i))\nonumber\\
&+&\left\{\frac{1}{N}\sum_{i=1}^N L(y_i,f_M(\Bx_i))- \frac{1}{N}\sum_{i=1}^N L(y_i,f_0(\Bx_i))\right\}\nonumber
\\&-&V (\hat{f}(\Bx_i)-f_0(\Bx_i))^2+\lambda \|f_0\|^2_k\nonumber\\
&=:&J_1+J_2+J_3+J_4, \label{J}
\end{eqnarray}
for some $V>0$ due to the strong convexity of $m_y(\cdot)$.
For the first term, we have
\begin{eqnarray*}
J_1&\leq& (\|\hat{f}\|_k+1) \sup_{f\in\mathcal{H}_k}\frac{1}{N}\sum_{i=1}^N -W_i \frac{f(\Bx_i)-f_0(\Bx_i)}{\|f\|_K+1}\\
&=&(\|\hat{f}\|_k+1) O_p(N^{-1/2}),
\end{eqnarray*}
where the last step follows from the fact that $\mathbb{E}W_i=0$, $W_i$ is bounded, and Lemma 3.4.3 of \cite{van1996weak} and the fact that $\|f-f_0\|_k/(\|f\|_k+1)=O(1)$. Clearly, we have $J_2=I_3+O_p(N^{-1/2})=O_p(N^{-1/2})$ according to the central limit theorem. The third term is clearly non-positive. We also have $J_4=O_p(N^{-1/2})$ by assumption for $\lambda$.

Now we conclude from (\ref{J}) that
$$\lambda \|\hat{f}\|^2_k\leq \|\hat{f}\|_k O_p(N^{-1/2})+O_p(N^{-1/2}), $$
which implies $\|\hat{f}\|_k=O_p(1)$. This completes the proof.
\bibliographystyle{IEEEtran}
\bibliography {bibliography.bib}

\end{document}

%% file: header.tex
\newcommand\numberthis{\addtocounter{equation}{1}\tag{\theequation}}

\newcommand{\gammab}{\boldsymbol{\gamma}}

\newtheorem{theorem}{Theorem}

\newtheorem{corollary}[theorem]{Corollary}

\newtheorem{lemma}[theorem]{Lemma}

\newcommand{\Real}{\mathbb R}

\newcommand{\NatInt}{\mathbb N}

\newcommand{\CalO}{\mathcal O}

\newcommand{\CalH}{\mathcal H}

\newcommand{\BX}{\bold X}
\newcommand{\Bx}{\bold x}
\newcommand{\By}{\bold y}
\newcommand{\BZ}{\bold Z}
\newcommand{\Bz}{\bold z}

\newtheorem{condition}{Condition}
\newtheorem{Assumption}{Assumption}